\documentclass[10pt,twocolumn,letterpaper]{article}

\usepackage{iccv}
\usepackage{times}
\usepackage{epsfig}
\usepackage{dsfont}
\usepackage{graphicx}
\usepackage{amsmath}
\usepackage{amssymb}
\usepackage{booktabs}
\usepackage{color}
\usepackage{ cite }

\usepackage{thmtools} 
\usepackage{thm-restate}
\usepackage{lipsum}

\usepackage[utf8]{inputenc}
\usepackage[english]{babel}

\usepackage{amsthm}
\newtheorem{theorem}{Theorem}[section]

 \normalsize

\usepackage[pagebackref=true,breaklinks=true,colorlinks,bookmarks=false]{hyperref}
\newcommand{\beginsupplement}{%
        \setcounter{table}{0}
        \renewcommand{\thetable}{S\arabic{table}}%
        \setcounter{figure}{0}
        \renewcommand{\thesection}{\Alph{section}}
     }

\iccvfinalcopy 

\def\iccvPaperID{8292} 

\ificcvfinal\fi

\begin{document}

\title{Weakly-supervised Video Anomaly Detection with Robust Temporal Feature Magnitude Learning} 

\author{\parbox{0.7\linewidth}{\centering Yu Tian$^{1,3}$     $\quad$ Guansong Pang$^1$ $\quad$ Yuanhong Chen$^1$   $\quad$  Rajvinder Singh$^3$ $\quad$ Johan W. Verjans$^{1,2,3}$  $\quad$ Gustavo Carneiro$^1$  $\newline$ $^{1}$ Australian Institute for Machine Learning, University of Adelaide \\
 $^{2}$ Faculty of Health and Medical Sciences, University of Adelaide
 \\
  $^{3}$ South Australian Health and Medical Research Institute \\
} 
}
\maketitle
\ificcvfinal\thispagestyle{empty}\fi

\begin{abstract}
Anomaly detection with weakly supervised video-level labels is typically formulated as a multiple instance learning (MIL) problem, in which we aim to identify snippets containing abnormal events, with each video represented as a bag of video snippets. 
Although current methods show effective detection performance, their recognition of the positive instances, i.e., rare abnormal snippets in the abnormal videos, is largely biased by the dominant negative instances, especially when the abnormal events are subtle anomalies that exhibit only small differences compared with normal events.
This issue is exacerbated in many methods that ignore important video temporal dependencies. To address this issue, we introduce a novel and theoretically sound method, named Robust Temporal Feature Magnitude learning (RTFM), which trains a feature magnitude learning function to effectively recognise the positive instances, substantially improving the robustness of the MIL approach to the negative instances from abnormal videos. RTFM also adapts dilated convolutions and self-attention mechanisms to capture long- and short-range temporal dependencies to learn the feature magnitude more faithfully. 
Extensive experiments show that the RTFM-enabled MIL model (i) outperforms several state-of-the-art methods by a large margin on four benchmark data sets (ShanghaiTech, UCF-Crime, XD-Violence and UCSD-Peds) and (ii) achieves significantly improved subtle anomaly discriminability and sample efficiency.

\end{abstract}

\section{Introduction}
Video anomaly detection has been intensively studied because of its potential to be used in autonomous surveillance systems~\cite{sultani2018real,hasan2016learning,Wu2020not,zhong2019graph}.
\begin{figure}[t]
\begin{center}
\small
  \includegraphics[width=1.0\linewidth]{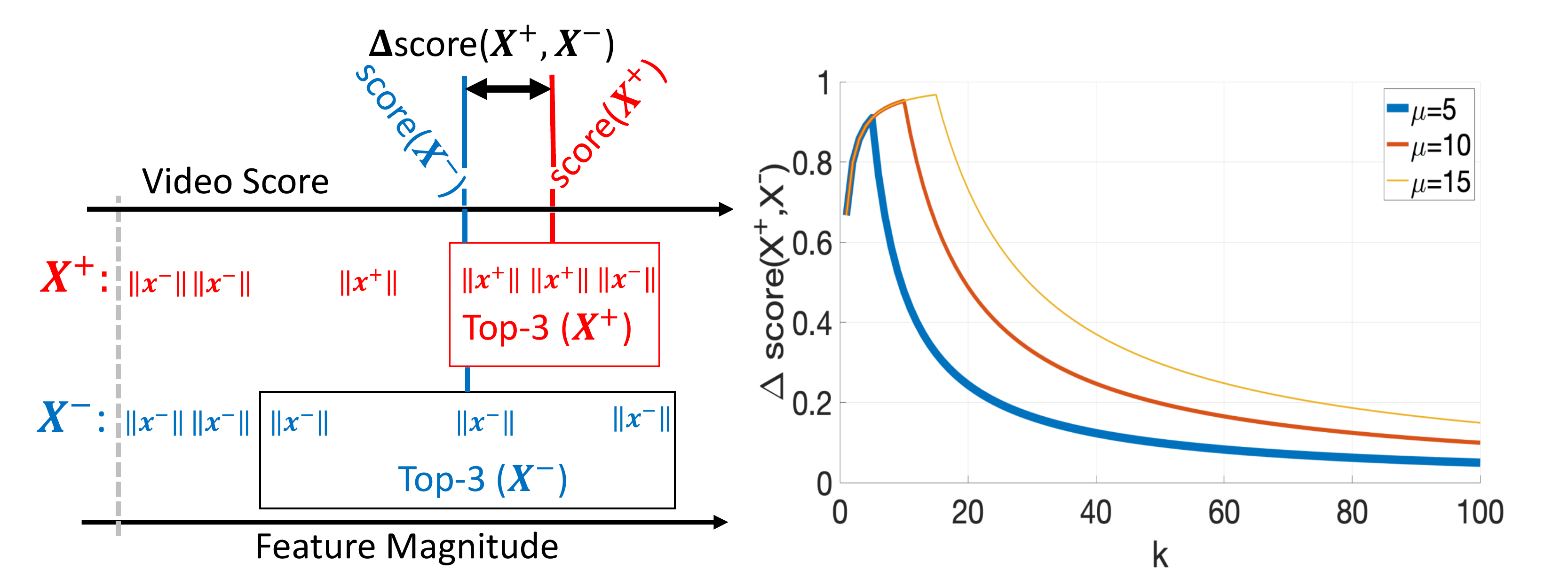}
\end{center}
  \caption{\textbf{RTFM} trains a feature magnitude learning function to improve the robustness of MIL approaches to normal snippets from abnormal videos, and detect  abnormal snippets more effectively. \textbf{Left:} temporal feature magnitudes of abnormal and normal snippets ($\|\mathbf{x}^+\|$ and $\|\mathbf{x}^-\|$), from abnormal and normal videos (\textcolor{red}{$\mathbf{X}^+$} and \textcolor{blue}{$\mathbf{X}^-$}). Assuming that $\mu=3$ denotes the number of abnormal snippets in the anomaly video,  
  we can maximise the $\Delta$score$(\mathbf{X}^+,\mathbf{X}^-)$, which measures the difference between the scores of abnormal and normal videos, by selecting the top $k \le \mu$ snippets with the largest temporal feature magnitude (the scores are computed with the mean of magnitudes of the top $k$ snippets). \textbf{Right:}  the $\Delta$score$(\mathbf{X}^+,\mathbf{X}^-)$ increases with $k\in[1,\mu]$ and then decreases for $k>\mu$, showing evidence that our proposed RTFM-enabled MIL model provides a better separation between abnormal and normal videos when $k \approx \mu$, even if there are a few normal snippets with large feature magnitudes.}
\label{fig:intro}

\end{figure}
The goal of video anomaly detection is to identify the time window when an anomalous event happened -- in the context of surveillance, examples of anomaly are bullying, shoplifting, violence, etc. 
Although one-class classifiers (OCCs, also called unsupervised anomaly detection) trained exclusively with normal videos have been explored in this context~\cite{hasan2016learning,zhang2016video,ravanbakhsh2017abnormal,ravanbakhsh2018plug,luo2017revisit,liu2018future, hinami2017joint}, the best performing approaches explore  
a weakly-supervised setup using training samples with \textit{video-level} label annotations of normal or abnormal~\cite{Wu2020not,zhong2019graph,sultani2018real}. This weakly-supervised setup
targets a better anomaly classification accuracy at the expense of a relatively small human annotation effort, compared with OCC approaches.

One of the major challenges of weakly supervised anomaly detection is how to identify anomalous snippets from a whole video labelled as abnormal. 
This is due to two reasons, namely: 1) the majority of snippets from an abnormal video consist of normal events, which can overwhelm the training process and challenge the fitting of the few abnormal snippets; and 2) abnormal snippets may not be sufficiently different from normal ones, making a clear separation between normal and abnormal snippets challenging.
Anomaly detection trained with multiple-instance learning (MIL) approaches~\cite{sultani2018real,Wu2020not,zhu2019motion,8803657} mitigates the issues above by balancing the training set with the same number of abnormal and normal snippets, where normal snippets are randomly selected from the normal videos and abnormal snippets are the ones with the top anomaly scores from abnormal videos. Although partly addressing the issues above, MIL introduces four problems: 1) the top anomaly score in an abnormal video may not be from an abnormal snippet; 2) normal snippets randomly selected from normal videos may be relatively easy to fit, which challenges training convergence; 3) if the video has more than one abnormal snippet, we miss the chance of having a more effective training process containing more abnormal snippets per video; and 4) the use of classification score provides a weak training signal that does not necessarily enable a good separation between normal and abnormal snippets.
These issues are exacerbated even more in methods that ignore important temporal dependencies~\cite{zhong2019graph,liu2018future,Wu2020not,luo2017revisit}.

To address the MIL problems above, we propose a novel method, named 
Robust Temporal Feature Magnitude (RTFM) learning. 
In RTFM, we rely on the temporal feature magnitude of video snippets, where features with low magnitude represent normal (i.e., negative) snippets and high magnitude features denote abnormal (i.e., positive)) snippets.
RTFM is theoretically motivated by the top-$k$ instance MIL~\cite{li2015multiple} that trains a classifier using $k$ instances with top classification scores from the abnormal and normal videos, but in our formulation, we assume that the mean feature magnitude 
of abnormal snippets is larger than that of normal snippets, instead of assuming separability between the classification scores of abnormal and normal snippets~\cite{li2015multiple}.
RTFM solves the MIL issues above, as follows: 1) the probability of selecting abnormal snippets from abnormal videos increases; 2) the hard negative normal snippets selected from the normal videos will be harder to fit, improving training convergence; 3) it is possible to include more abnormal snippets per abnormal video; and 4) using feature magnitude to recognise positive instances is advantageous compared to MIL methods that use classification scores~\cite{sultani2018real,li2015multiple}, because it enables a stronger learning signal, particularly for the abnormal snippets that have a magnitude that can increase for the whole training process, and the feature magnitude learning can be jointly optimised with the MIL anomaly classification to enforce large margins between abnormal and normal snippets at both the feature representation space and the anomaly classification output space. 
Fig.~\ref{fig:intro} motivates RTFM, showing that the selection of the top-$k$ features (based on their magnitude) can provide a better separation between abnormal and normal videos, when we have more than one abnormal snippet per abnormal video and the mean snippet feature magnitude of abnormal videos is larger than that of normal videos.

In practice, RTFM enforces large margins between the top $k$ snippet features with largest magnitudes from abnormal and normal videos, which has theoretical guarantees to maximally separate abnormal and normal video representations.
These top $k$ snippet features from normal and abnormal videos are then selected to train a snippet classifier. 
To seamlessly incorporate long and short-range temporal dependencies within each video, we combine the learning of long and short-range temporal dependencies with a pyramid of dilated convolutions (PDC)~\cite{yu2015multi} and a temporal self-attention module (TSA)~\cite{wang2018non}.
We validate our RTFM on four anomaly detection benchmark data sets, namely ShanghaiTech~\cite{liu2018future}, UCF-Crime~\cite{sultani2018real}, XD-Violence~\cite{Wu2020not} and UCSD-Peds~\cite{li2013anomaly}. We show that our method outperforms the current SOTAs by a large margin on all benchmarks using different pre-trained features (i.e., C3D and I3D). We also show that our method achieves substantially better sample efficiency and subtle anomaly discriminability than popular MIL methods.


\section{Related Work}

\textbf{Unsupervised Anomaly Detection.} Traditional anomaly detection methods assume the availability of normal training data only and address the problem with one-class classification using handcrafted features~\cite{medioni2001event,basharat2008learning,wang2014learning,zhang2009learning}. 
With the advent of deep learning, more recent approaches use the features from pre-trained deep neural networks~\cite{zhao2020exploring,smeureanu2017deep,pang2020self,tudor2017unmasking,fang2020anomaly}. Others apply constraints on the latent space of normal manifold to learn compact normality representations
~\cite{Markovitz_2020_CVPR,Bergmann_2020_CVPR,Park_2020_CVPR,Bergmann_2019_CVPR,Morais_2019_CVPR,Abati_2019_CVPR,Perera_2019_CVPR,bergman2020classification,zhou2020encoding,Sabokrou_2018_CVPR,ruff2018deep,golan2018deep,wang2019gods,del2016discriminative,Cheng_2015_CVPR,tian2021constrained,liu2020photoshopping,chen2021unsupervised,sun2020discriminative}. 
Alternatively, some approaches depend on data reconstruction using generative models to learn the representations of normal samples by (adversarially) minimising the reconstruction error~\cite{liu2018future,ren2015unsupervised,xu2015learning,ionescu2019object,gong2019memorizing,sabokrou2017deep,Sabokrou_2018_CVPR,morais2019learning,ionescu2019object,Park_2020_CVPR,Burlina_2019_CVPR,venkataramanan2019attention,zong2018deep,Nguyen_2019_ICCV,nguyen2019anomaly}. These approaches assume that unseen anomalous videos/images often cannot be reconstructed well and consider samples of high reconstruction errors to be anomalies.
However, due to the lack of prior knowledge of abnormality, these approaches can overfit the training data and fail to distinguish abnormal from normal events. Readers are referred to \cite{pang2021deep} for a comprehensive review of those anomaly detection approaches. 

\textbf{Weakly Supervised Anomaly Detection.} Leveraging some labelled abnormal samples has shown substantially improved performance over the unsupervised approaches~\cite{tian2020few,sultani2018real,Wu2020not,liu2019margin,ruff2019deep,pang2019deep,pang2018learning,zaheer2020claws,zaheer2021cleaning,zaheer2020self}.
However, large-scale frame-level label annotation is too expensive to obtain. Hence, current SOTA video anomaly detection approaches rely on weakly supervised training that uses cheaper video-level annotations. Sultani et al.~\cite{sultani2018real} proposed the use of video-level labels and introduced the large-scale weakly-supervised video anomaly detection data set, UCF-Crime. 
Since then, this direction has attracted the attention of the research community~\cite{8803657,9102722,Wu2020not}. 

Weakly-supervised video anomaly detection methods are mainly based on the MIL framework~\cite{sultani2018real}. However, most MIL-based methods~\cite{sultani2018real,zhu2019motion,8803657} fail to leverage abnormal video labels as they can be affected by the label noise in the positive bag caused by a normal snippet mistakenly selected as the top abnormal event in an anomaly video. 
To deal with this problem, Zhong et al.~\cite{zhong2019graph} reformulated this problem as a binary classification under noisy label problem and used a graph convolution neural (GCN) network to clear the label noise. Although this paper shows more accurate results than~\cite{sultani2018real}, the training of GCN and MIL is computationally costly, and it can lead to unconstrained latent space (i.e., normal and abnormal features can lie at any place of the feature space) that can cause unstable performance. By contrast, our method has trivial computational overheads compared to the original MIL formulation. Moreover, our method unifies the representation learning and anomaly score learning by an $\ell_2$-norm-based temporal feature ranking loss, enabling better separation between normal and abnormal feature representations, improving the exploration of weak labels compared to previous MIL methods~\cite{sultani2018real,8803657,9102722,Wu2020not,zhu2019motion,zhong2019graph}.

\section{The Proposed Method: RTFM}
\label{sec:architecture_overview}

Our proposed robust temporal feature magnitude (RTFM) approach aims to differentiate between abnormal and normal snippets using weakly labelled videos for training. 
Given a set of weakly-labelled training videos $\mathcal{D} = \{ (\mathbf{F}_i,y_i) \}_{i=1}^{|\mathcal{D}|}$, where $\mathbf{F} \in \mathcal{F} \subset \mathbb{R}^{T \times D}$ 
are pre-computed features (e.g., I3D~\cite{carreira2017quo} or C3D~\cite{tran2015learning}) of dimension $D$ from the $T$ video snippets, and $y \in \mathcal{Y} = \{0,1\}$ denotes the video-level annotation ($y_i=0$ if $\mathbf{F}_i$ is a normal video and $y_i=1$ otherwise).
The model used by RTFM is denoted by $r_{\theta,\phi}(\mathbf{F})=f_{\phi}(s_{\theta}(\mathbf{F}))$ and returns a $T$-dimensional feature $[0,1]^{T}$ representing the classification of the $T$ video snippets into abnormal or normal, with the parameters $\theta,\phi$ defined below.
The training of this model comprises a joint optimisation of an end-to-end multi-scale temporal feature learning, and feature magnitude learning and an RTFM-enabled MIL classifier training, with the loss
\begin{equation}
    \min_{\theta, \phi} \sum_{i,j=1}^{|\mathcal{D}|} 
    \ell_{s}(s_{\theta}(\mathbf{F}_i),(s_{\theta}(\mathbf{F}_j)),y_i,y_j) + 
    \ell_{f}(f_{\phi}(s_{\theta}(\mathbf{F}_i)),y_i),
    \label{eq:main_loss}
\end{equation}
where 
$s_{\theta}:\mathcal{F} \to \mathcal{X}$ is the temporal feature extractor (with $\mathcal{X} \subset \mathbb{R}^{T \times D}$), 
$f_{\phi}:\mathcal{X} \to [0,1]^T$
is the snippet classifier, 
$\ell_{s}(.)$ denotes a loss function 
that maximises the separability between the top-$k$ snippet features from normal and abnormal videos, and $\ell_{f}(.)$ is a loss function to train the snippet classifier $f_{\phi}(.)$ also using the top-$k$ snippet features from normal and abnormal videos. 
Next, we discuss the theoretical motivation for our proposed RTFM, followed by a detailed description of the approach.

\subsection{Theoretical Motivation of RTFM}
\label{sec:theory_RTFM}

Top-$k$ MIL in~\cite{li2015multiple} extends MIL to an environment where positive bags contain a minimum number of positive samples and negative bags also contain positive samples, but to a lesser extent, and it assumes that a classifier can separate positive and negative samples.
Our problem is different because negative bags do not contain positive samples, and we do not make the classification separability assumption.
Following the nomenclature introduced above, a temporal feature extracted from a video is denoted by $\mathbf{X} = s_{\theta}(\mathbf{F})$ in~\eqref{eq:main_loss}, where snippet features are represented by the rows $\mathbf{x}_t$ of $\mathbf{X}$.
An abnormal snippet is denoted by $\mathbf{x}^+ \sim P_{x}^+(\mathbf{x})$, and a normal snippet, $\mathbf{x}^- \sim P_{x}^-(\mathbf{x})$.
An abnormal video $\mathbf{X}^+$ contains $\mu$ snippets drawn from $P_{x}^+(\mathbf{x})$ and $(T-\mu)$ drawn from $P_{x}^-(\mathbf{x})$, and a normal video $\mathbf{X}^-$ has all $T$ snippets sampled from $P_{x}^-(\mathbf{x})$. 

To learn a function that can classify videos and snippets as normal or abnormal, we define a function that classifies a snippet using its magnitude (i.e., we use $\ell_2$ norm to compute the feature magnitude),
where instead of assuming classification separability between normal and abnormal snippets (as assumed in~\cite{li2015multiple}), we make a milder assumption that 
$\mathbb{E}[\| \mathbf{x}^+ \|_2] \ge \mathbb{E}[\| \mathbf{x}^- \|_2]$.  
This means that by learning the snippet feature from $s_\theta(\mathbf{F})$, such that normal ones have smaller feature magnitude than abnormal ones, we can satisfy this assumption. To enable such learning, we rely on an optimisation based on the mean feature magnitude of the top $k$ snippets from a video~\cite{li2015multiple}, 
defined by 
\begin{equation}
    g_{\theta,k}(\mathbf{X}) = \max_{\Omega_k(\mathbf{X}) \subseteq \{ \mathbf{x}_t\}_{t=1}^T} \frac{1}{k}\sum_{\mathbf{x}_t \in \Omega_k(\mathbf{X})} \| \mathbf{x}_t \|_2,
\label{eq:score_top_k_instances}    
\end{equation}
where $g_{\theta,k}(.)$ is parameterised by $\theta$ to indicate its dependency on $s_{\theta}(.)$ to produce $\mathbf{x}_t$,
$\Omega_k(\mathbf{X})$ contains a subset of $k$ snippets from $\{ \mathbf{x}_t\}_{t=1}^T$ and $|\Omega_k(\mathbf{X})|=k$. 
The separability between abnormal and normal videos is denoted by
\begin{equation}
    d_{\theta,k}(\mathbf{X}^+,\mathbf{X}^-)=  g_{\theta,k}(\mathbf{X}^{+}) -  g_{\theta,k}(\mathbf{X}^{-}).
    \label{eq:delta_separability}
\end{equation}
For the theorem below, we define the probability that a snippet from $\Omega_{k}(\mathbf{X}^+)$ is abnormal with $p^+_{k}(\mathbf{X}^+) = \frac{\min(\mu,k)}{k + \epsilon}$, with $\epsilon > 0$ and from normal  $\Omega_{k}(\mathbf{X}^-)$, $p^+_{k}(\mathbf{X}^-) = 0$.  This definition means that it is likely to find an abnormal snippet within the top $k$ snippets in $\Omega_{k}(\mathbf{X}^+)$, as long as $k \le \mu$.

\begin{theorem}[Expected Separability Between Abnormal and Normal  Videos]
\label{thm:expected_separability}
Assuming that $\mathbb{E}[\| \mathbf{x}^+ \|_2] \ge \mathbb{E}[\| \mathbf{x}^- \|_2]$, where $\mathbf{X}^+$ has $\mu$ abnormal samples and $(T-\mu)$ normal samples, where $\mu \in [1,T]$, and $\mathbf{X}^-$ has $T$ normal samples.
Let $D_{\theta,k}(.)$ be the random variable from which the separability scores $d_{\theta,k}(.)$ of~\eqref{eq:delta_separability} are drawn~\cite{li2015multiple}.
\begin{enumerate}
    \item If  $0 < k < \mu$, then
    $$0 \le \mathbb{E}[D_{\theta,k}(\mathbf{X}^+,\mathbf{X}^-)] \le \mathbb{E}[D_{\theta,k+1}(\mathbf{X}^+,\mathbf{X}^-)].$$
    \item For a finite $\mu$, then
    $$ \lim_{k \to \infty} \mathbb{E}[D_{\theta,k}(\mathbf{X}^+,\mathbf{X}^-)] = 0.$$
\end{enumerate}
\end{theorem}
\begin{proof}
    Please see proof in the supplementary material.
\end{proof}
Therefore, the first part of this theorem means that as we include more samples in the top $k$ snippets of the abnormal video, 
the separability between abnormal and normal video tends to increase (even if it includes a few normal samples) as long as $k \le \mu$.  
The second part of the theorem means that as we include more than $\mu$ top instances, the abnormal and normal video scores become indistinguishable because of the overwhelming number of negative samples both in the positive and negative bags.  
Both points are shown in Fig.~\ref{fig:intro}, where score($\mathbf{X}$)=$g_{\theta,k}(\mathbf{X})$, $\Delta$score($\mathbf{X}^+,\mathbf{X}^-$) = $d_{\theta,k}(\mathbf{X}^+,\mathbf{X}^-)$, and $\epsilon=0.4$ to compute $p^+_{k}(\mathbf{X^+})$. This theorem suggests that by maximising the separability of the top-$k$ temporal feature snippets from abnormal and normal videos (for $k \le \mu)$, we can facilitate the classification of anomaly videos and snippets.  It also suggests that the use of the top-$k$ features to train the snippet classifier allows for a more effective training given that the majority of the top-$k$ samples in the abnormal video will be abnormal and that we will have a balanced training using the top-$k$ hardest normal snippets. The final consideration is that because we use just the top-$k$ samples per video, our method is efficiently optimised with a relatively small amount of training samples. 

\begin{figure}
\begin{center}
\includegraphics[width=1.0\linewidth]{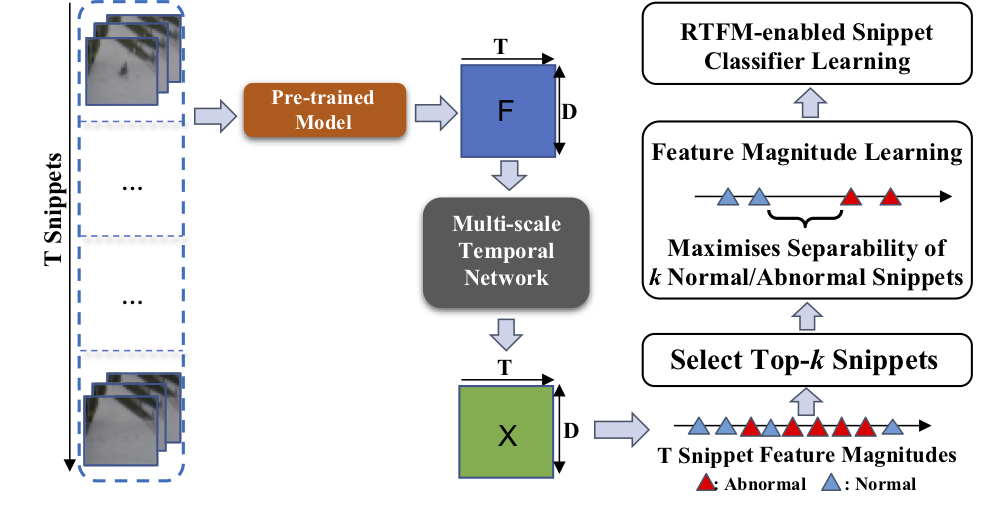}
\end{center}
  \caption{Our proposed RTFM receives a $T \times D$ feature matrix $\mathbf{F}$ extracted from a video containing $T$ snippets. Then, MTN captures the long and short-range temporal dependencies between snippet features to produce $\mathbf{X}=s_{\theta}(\mathbf{F})$. Next, we maximise the separability between abnormal and normal video features and train a snippet classifier using the top-$k$ largest magnitude feature snippets from abnormal and normal videos. 
  }
\label{fig:framework}
\end{figure}

\subsection{Multi-scale Temporal Feature Learning}
\label{sec:temporal}
Inspired by the attention techniques used in video understanding~\cite{wang2018non,8451429}, our proposed multi-scale temporal network (MTN) captures the multi-resolution local temporal dependencies and the global temporal dependencies between video snippets (we depict MTN in Fig.1 of the supplementary material). MTN uses a pyramid of dilated convolutions over the time domain to learn multi-scale representations for video snippets. 
Dilated convolution is usually applied in the spatial domain with the goal of expanding the receptive field without losing resolution~\cite{yu2015multi}. Here we propose to use dilated convolutions over the temporal dimension as it is important to capture the multi-scale temporal dependencies of neighbouring video snippets for anomaly detection.


MTN learns the multi-scale temporal features from the pre-computed fetures
$\mathbf{F}=[\mathbf{f}_d]_{d=1}^D$.
Then given the feature $\mathbf{f}_d \in \mathbb{R}^{T}$, the 1-D dilated convolution operation with kernel $\mathbf{W}^{(l)}_{k,d} \in \mathbb{R}^{W}$ with $k \in \{1,...,D/4\}$, $d \in \{1,...,D\}$, $l \in \{ \text{PDC}_1,\text{PDC}_2,\text{PDC}_3 \}$, and $W$ denoting the filter size, is defined by
\begin{equation}
   \mathbf{f}^{(l)}_k = \sum_{d=1}^{D} \mathbf{W}^{(l)}_{k,d} *^{(l)} \mathbf{f}_d,
    \label{eq:dilated}
\end{equation}
where $*^{(l)}$ represents the dilated convolution operator indexed by $l$,
$\mathbf{f}^{(l)}_k \in \mathbb{R}^{T}$  represents the output features after applying the dilated convolution over the temporal dimension. 
The dilation factors for $\{ \text{PDC}_1,\text{PDC}_2,\text{PDC}_3 \}$ are $\{1,2,4\}$, respectively (this is shown in Fig.1 of the supplementary material).

The global temporal dependencies between video snippets is achieved with a self-attention module, which has shown promising performance on capturing the long-range spatial dependency on video understanding~\cite{wang2018non}, image classification~\cite{zhao2020exploring} and object detection~\cite{perreault2020spotnet}. 
Motivated by the previous works using GCN to model global temporal information~\cite{zhong2019graph,Wu2020not}, we re-formulate the spatial self-attention technique to work on the time dimension and capture global temporal context modelling. In detail, we aim to produce an attention map $\mathbf{M} \in \mathbb{R}^{T \times T}$ that estimates the pairwise correlation between snippets. 
Our temporal self-attention (TSA) module first uses a $1 \times 1$ convolution to reduce the spatial dimension from $\mathbf{F} \in \mathbb R^{T \times D}$ to $\mathbf{F}^{(\text{c})} \in \mathbb{R}^{T \times D/4}$ with $\mathbf{F}^{(\text{c})}=Conv_{1\times 1}(\mathbf{F})$. 
We then apply three separate $1 \times 1$ convolution layers to $\mathbf{F}^{(\text{c})}$ to produce 
$\mathbf{F}^{(\text{c1})},\mathbf{F}^{(\text{c2})},\mathbf{F}^{(\text{c3})} \in \mathbb{R}^{T \times D/4}$, as in
$\mathbf{F}^{(\text{ci})}=Conv_{1\times 1}(\mathbf{F}^{(\text{c})})$ for $i \in \{1,2,3\}$.
The attention map is then built with $\mathbf{M} = \left ( \mathbf{F}^{(\text{c1})} \right ) \left ( \mathbf{F}^{(\text{c2})} \right )^{\intercal}$, which 
produces $\mathbf{F}^{(\text{c4})} = Conv_{1\times 1}(\mathbf{M}\mathbf{F}^{(\text{c3})})$.

A skip connection is added after this final $1 \times 1$ convolutional layer, as in
\begin{equation}
    \mathbf{F}^{(\text{TSA})} = \mathbf{F}^{(\text{c4})} + \mathbf{F}^{(\text{c})}. 
    \label{eq:sa}
\end{equation}

The output from the MTN is formed with a concatenation of the outputs from the PDC and MTN modules $\bar{\mathbf{F}} = [\mathbf{F}^{(l)}]_{l \in \mathcal{L}} \in \mathbb R^{T \times D}$, with $\mathcal{L}=\{ \text{PDC}_{1}, \text{PDC}_{2},\text{PDC}_{3},\text{TSA} \}$.  A skip connection using the original features $\mathbf{F}$ produces the final temporal feature representation $\mathbf{X} = s_{\theta}(\mathbf{F}) = \bar{\mathbf{F}} + \mathbf{F}$,
where the parameter $\theta$ comprises the weights for all convolutions described in this section.

\subsection{Feature Magnitude Learning}
\label{sec:top_k_MIL} 

Using the theory introduced in Sec.~\ref{sec:theory_RTFM}, we propose a loss function to model $s_{\theta}(\mathbf{F})$ in~\eqref{eq:main_loss}, where the top $k$ largest snippet feature magnitudes from normal videos are minimised and the top $k$ largest snippet feature magnitudes from abnormal videos are maximised.
More specifically, we propose the following loss $\ell_s(.)$ from~\eqref{eq:main_loss} that maximises the separability between normal and abnormal videos:
\begin{equation}
\begin{split}
   \ell_{s}&(s_{\theta}(\mathbf{F}_i),s_{\theta}(\mathbf{F}_j),y_i,y_j)  = \\
&\left\{
	\begin{array}{ll}
   \max \Big ( 0, m - d_{\theta,k}(\mathbf{X}_i,\mathbf{X}_j)\Big ) &, \text{if } y_i=1,y_j=0 \\
   0 &, \text{otherwise}
   \end{array}
   \right.
\end{split}   
    \label{eq:L2C}
\end{equation}
where $m$ is a pre-defined margin, $\mathbf{X}_i=s_{\theta}(\mathbf{F}_i)$ is the abnormal video feature (similarly for $\mathbf{X}_j$ for a normal video),
and $d_{\theta,k}(.)$ represents separability function defined in~\eqref{eq:delta_separability} that computes the difference between 
the score of the top $k$ instances, from $g_{\theta,k}(.)$ in~\eqref{eq:score_top_k_instances}, of the abnormal and normal videos.

\subsection{RTFM-enabled Snippet Classifier Learning}

To learn the snippet classifier, we train a binary cross-entropy-based classification loss function using the set $\Omega_{k}(\mathbf{X})$ that contains the $k$ snippets with the largest $\ell_2$-norm features from $s_{\theta}(\mathbf{F})$ in~\eqref{eq:main_loss}.
In particular, the loss $\ell_f(.)$ from~\eqref{eq:main_loss} is defined as
\begin{equation}
\begin{split}
  &\ell_{f}(f_{\phi}(s_{\theta}(\mathbf{F})),y) =\\
  &\sum_{\mathbf{x} \in \Omega_{k}(\mathbf{X})} -(y\log(f_{\phi}(\mathbf{x})) + (1-y)\log(1-f_{\phi}(\mathbf{x}))),
\end{split}
    \label{eq:BCE}
\end{equation}
where $\mathbf{x}=s_{\theta}(\mathbf{f})$.
Note that following~\cite{sultani2018real}, $\ell_{f}(.)$ is accompanied by the temporal smoothness and sparsity regularisation, with the temporal smoothness defined as $\big(f_{\phi}(s_{\theta}(\mathbf{f}_t))-f_{\phi}(s_{\theta}(\mathbf{f}_{t-1}))\big)^{2}$ to enforce similar anomaly score for neighbouring snippets, while the sparsity regularisation defined as $\sum_{t = 1}^T |f_{\phi}(s_{\theta}(\mathbf{f}_t))|$ to impose a prior that abnormal events are rare in each abnormal video.
\section{Experiments}

\subsection{Data Sets and Evaluation Measure}

Our model is evaluated on four multi-scene benchmark datasets, created for the weakly supervised video anomaly detection task: ShanghaiTech~\cite{liu2018future}, UCF-Crime~\cite{sultani2018real}, XD-Violence~\cite{Wu2020not} and UCSD-Peds~\cite{xu2014video}.

\textbf{UCF-Crime} is a large-scale anomaly detection data set~\cite{sultani2018real} that contains 1900 untrimmed videos with a total duration of 128 hours from real-world street and indoor surveillance cameras. Unlike the static backgrounds in ShanghaiTech, UCF-Crime consists of complicated and diverse backgrounds.
Both training and testing sets contain the same number of normal and abnormal videos. The data set covers 13 classes of anomalies in 1,610 training videos with video-level labels and 290 test videos with frame-level labels. 

\textbf{XD-Violence} is a recently proposed large-scale multi-scene anomaly detection data set, collected from real life movies, online videos, sport streaming, surveillance cameras and CCTVs~\cite{Wu2020not}. The total duration of this data set is over 217 hours, containing 4754 untrimmed videos with video-level labels in the training set and frame-level labels in the testing set. It is currently the largest publicly available video anomaly detection data set.

\textbf{ShanghaiTech} is a medium-scale data set from fixed-angle street video surveillance. It has 13 different background scenes and 437 videos, including 307 normal videos and 130 anomaly videos. The original data set~\cite{liu2018future} is a popular benchmark for the anomaly detection task that assumes the availability of normal training data.
Zhong et al.~\cite{zhong2019graph} reorganised the data set by selecting a subset of anomalous testing videos into training data to build a weakly supervised training set,
so that both training and testing sets cover all 13 background scenes. We use exactly the same procedure as in~\cite{zhong2019graph} to convert ShanghaiTech for the weakly supervised setting.

\textbf{UCSD-Peds} is a small-scale dataset combined by two sub-datasets -- Ped1 with 70 videos and Peds2 with 28 videos. Previous work~\cite{zhong2019graph,he2018anomaly} re-formulate the dataset for weakly supervised anomaly detection by randomly selecting 6 anomaly videos and 4 normal videos into the train set, with the remaining as test set. We report the mean results over 10 times of this process.


\textbf{Evaluation Measure.} Similarly  to previous papers~\cite{sultani2018real,liu2018future,8803657,9102722,gong2019memorizing}, we use the frame-level area under the ROC curve (AUC) as the evaluation measure for all data sets. 
Moreover, following~\cite{Wu2020not}, we also use average precision (AP) as the evaluation measure for the XD-Violence data set. Larger AUC and AP values indicate better performance. Some recent studies~\cite{georgescu2020anomaly,ramachandra2020survey} recommend using the region-based detection criterion (RBDC) and the track-based detection criterion (TBDC) to complement the AUC measure, but these two measures are inapplicable in the weakly-supervised setting. Thus, we focus on the AUC and AP measures.

\subsection{Implementation Details} \label{subsec:imp}
Following~\cite{sultani2018real}, each video is divided into 32 video snippets, i.e., $T=32$.
For all experiments, we set the margin $m = 100$, $k = 3$ in~\eqref{eq:L2C}.
The three FC layers described in the model (Sec.~\ref{sec:architecture_overview}) have 512, 128 and 1 nodes, where each of those FC layers is followed by a ReLU activation function and a dropout function with a dropout rate of 0.7. The 2048D and 4096D features are extracted from the '$mix\_5c$' and '$fc\_6$' layer of the pre-trained I3D~\cite{kay2017kinetics} or C3D~\cite{KarpathyCVPR14} network, respectively. In MTN, we set the pyramid dilate rate as 1, 2 and 4, and we use the 3 $\times $ 1 Conv1D for each dilated convolution branch. For the self-attention block, we use a 1 $\times $ 1 Conv1D. 


Our RTFM method is trained in an end-to-end manner using the Adam optimiser~\cite{kingma2014adam} with a weight decay of 0.0005 and a batch size of 64 for 50 epochs. 
The learning rate is set to 0.001 for ShanghaiTech and UCF-Crime, and 0.0001 for XD-Violence. Each mini-batch consists of samples from 32 randomly selected normal and abnormal videos. The method is implemented using PyTorch~\cite{NEURIPS2019_9015}. 
For all baselines, we use the published results with the same backbone as ours. For a fair comparison, we use the same benchmark setup as in~\cite{sultani2018real,Wu2020not,zhong2019graph}.

\subsection{Results on ShanghaiTech}
The frame-level AUC results on ShanghaiTech are shown in Tab.~\ref{tab:sh_tech_table}. Our method RTFM achieves superior performance when compared with previous SOTA unsupervised learning methods~\cite{hasan2016learning,luo2017revisit,liu2018future,Park_2020_CVPR,yu2020cloze} and weakly-supervised approaches~\cite{9102722,8803657,zhong2019graph}. With I3D-RGB features, our model obtains the best AUC result on this data set: 97.21\%. Using the same I3D-RGB features, our RTFM-enabled MIL method outperforms current SOTA MIL-based methods~\cite{sultani2018real,8803657,9102722} by 10\% to 14\%. Our model outperforms~\cite{9102722} by more than 5\% even though they rely on a more advanced feature extractor (i.e., I3D-RGB and I3D Flow). These results demonstrate the gains achieved from our proposed feature magnitude learning.

Our method also outperforms the GCN-based weakly-supervised method~\cite{zhong2019graph} by  11.7\%, which indicates that our MTN module is more effective at capturing temporal dependencies than GCN. 
Additionally, considering the C3D-RGB features, our model achieves the SOTA AUC of 91.51\%, significantly surpassing the previous methods with C3D-RGB by a large margin.

 

\begin{table}[htbp]
\centering
\scalebox{0.7}{
\begin{tabular}{@{}c|c|c|c@{}}
\toprule\hline
Supervision       & Method      & Feature             & AUC(\%) \\ \hline\hline
                  & Conv-AE~\cite{hasan2016learning}     & -                   & 60.85   \\
                  & Stacked-RNN~\cite{luo2017revisit} & -                   & 68.00      \\
Unsupervised      & Frame-Pred~\cite{liu2018future}  & -                   & 73.40    \\
                  & Mem-AE~\cite{gong2019memorizing}      & -                   & 71.20    \\
                  & MNAD~\cite{Park_2020_CVPR}        & -                   & 70.50    \\
                  & VEC~\cite{yu2020cloze}         & -                   & 74.80    \\ \hline
                
                  & GCN-Anomaly~\cite{zhong2019graph} & C3D-RGB             & 76.44   \\
                  & GCN-Anomaly~\cite{zhong2019graph} & TSN-Flow            & 84.13   \\
                  & GCN-Anomaly~\cite{zhong2019graph} & TSN-RGB             & 84.44   \\
                  & Zhang et al.~\cite{8803657}         & I3D-RGB             & 82.50    \\
                  & Sultani et al.*~\cite{sultani2018real} & I3D RGB  & 85.33    \\
Weakly Supervised & AR-Net~\cite{9102722}      & I3D Flow            & 82.32   \\
                  & AR-Net~\cite{9102722}      & I3D-RGB             & 85.38   \\
                  & AR-Net~\cite{9102722}      & I3D-RGB \& I3D Flow & 91.24   \\
                  & Ours        & C3D-RGB             & \textcolor{blue}{\textbf{91.51}}   \\
                  & Ours        & I3D-RGB             & \textcolor{red}{\textbf{97.21}}  \\ \hline\bottomrule
\end{tabular}%
}
\caption{Comparison of frame-level AUC performance with other SOTA un/weakly-supervised methods on ShanghaiTech. 
* indicates we retrain the method in~\cite{sultani2018real} using I3D features. Best result in \textcolor{red}{\textbf{red}} and second best in \textcolor{blue}{\textbf{blue}}.} 
\label{tab:sh_tech_table}

\end{table}

\subsection{Results on UCF-Crime}

The AUC results on UCF-Crime are shown in Tab.~\ref{tab:ucf-crime}. Our method outperforms all previous unsupervised learning approaches~\cite{hasan2016learning,sohrab2018subspace,luo2017revisit,wang2019gods}. Remarkably, using the same I3D-RGB features, our method also  outperforms current SOTA MIL-based methods, Sultani et al.~\cite{sultani2018real} by 8.62\%, Zhang et al.~\cite{8803657} by 5.37\%, Zhu et al.~\cite{zhu2019motion} by 5.03\% and Wu et al.~\cite{Wu2020not} by 1.59\%.  Zhong et al.~\cite{zhong2019graph} use a computationally costly alternating training scheme
to achieve an AUC of 82.12\%, while our method utilises an efficient end-to-end training scheme
and outperforms their approach by 1.91\%. Our method also surpasses the current SOTA unsupervised methods, BODS and GODS~\cite{wang2019gods}, by at least 13\%. 
Considering the C3D features, our method surpasses the previous weakly supervised methods by a minimum 2.95\% and a maximum 7.87\%, indicating the effectiveness of our RTFM approach regardless of the backbone structure.

\begin{table}[htbp]
\centering
\scalebox{0.7}{
\begin{tabular}{@{}c|c|c|c@{}}
\toprule\hline
Supervision       & Method         & Feature  & AUC (\%) \\ \hline\hline
                  & SVM Baseline   & -        & 50.00       \\
                  & Conv-AE~\cite{hasan2016learning}   & -        & 50.60       \\
                  & Sohrab et al.~\cite{sohrab2018subspace}  & -        & 58.50     \\
Unsupervised      & Lu et al.~\cite{lu2013abnormal}  & C3D RGB        & 65.51    \\
                  & BODS~\cite{wang2019gods}           & I3D RGB  & 68.26    \\
                  & GODS~\cite{wang2019gods}           & I3D RGB  & 70.46    \\ \hline
                  & Sultani et al.~\cite{sultani2018real} & C3D RGB  & 75.41    \\
                  & Sultani et al.*~\cite{sultani2018real} & I3D RGB  & 77.92    \\
                  & Zhang et al.~\cite{8803657}            & C3D RGB  & 78.66    \\
                  & Motion-Aware~\cite{zhu2019motion} & PWC Flow & 79.00       \\
                  & GCN-Anomaly~\cite{zhong2019graph}    & C3D RGB  & 81.08    \\
Weakly Supervised & GCN-Anomaly~\cite{zhong2019graph}    & TSN Flow & 78.08    \\
                  & GCN-Anomaly~\cite{zhong2019graph}    & TSN RGB  & 82.12    \\
                  & Wu et al.~\cite{Wu2020not}      & I3D RGB  & 82.44    \\
                  & Ours           & C3D RGB  & \textcolor{blue}{\textbf{83.28}}   \\
                  & Ours           & I3D RGB  & \textcolor{red}{\textbf{84.30}}   \\ \hline\bottomrule
\end{tabular}%
}
\caption{Frame-level AUC performance on UCF-Crime. * indicates we retrain the method in~\cite{sultani2018real} using I3D features. Best result in \textcolor{red}{\textbf{red}} and second best in \textcolor{blue}{\textbf{blue}}.} 
\label{tab:ucf-crime}
\end{table}


\subsection{Results on XD-Violence}

XD-Violence is a recently released data set, on which few results have been reported, as displayed in Tab.~\ref{tab:xd-violence}. Our approach surpasses all unsupervised learning approaches by a minimum of 27.03\% in AP. 
Comparing with SOTA weakly-supervised methods~\cite{Wu2020not,sultani2018real}, our method is 2.4\% and 2.13\% better than Wu et al.~\cite{Wu2020not} and Sultani et al.~\cite{sultani2018real}, using the same I3D features. With the C3D features, our RTFM achieves the best 75.89\% AUC when compared with the MIL baseline by Sultani et al.~\cite{sultani2018real}. The consistent superiority of our method reinforces the effectiveness of our proposed feature magnitude learning method in enabling the MIL-based anomaly classification.

\begin{table}[htbp]
\centering
\scalebox{0.8}{
\begin{tabular}{@{}c|c|c|c@{}}
\toprule\hline
Supervision       & Method         & Feature          & AP(\%) \\ \hline\hline
                  & SVM baseline   & -                & 50.78  \\
Unsupervised      & OCSVM~\cite{scholkopf2000support}         & -                & 27.25  \\
                  & Hasan et al.~\cite{hasan2016learning}  & -                & 30.77  \\ \hline
                  & Sultani et al.~\cite{sultani2018real} & C3D RGB          & 73.20   \\
Weakly Supervised & Sultani et al.*~\cite{sultani2018real} & I3D RGB          & 75.68   \\
                  & Wu et al.~\cite{Wu2020not}      & I3D RGB          & 75.41 \\
                  & Ours           & C3D RGB          & \textcolor{blue}{\textbf{75.89}} \\
                  & Ours           & I3D RGB          & \textcolor{red}{\textbf{77.81}}  \\\hline\bottomrule
                 
\end{tabular}%
}
\caption{Comparison of AP performance with other SOTA un/weakly-supervised methods on XD-Violence. * indicates we retrain the method in~\cite{sultani2018real} using I3D features. Best result in \textcolor{red}{\textbf{red}} and second best in \textcolor{blue}{\textbf{blue}}.}
\label{tab:xd-violence}
\end{table}

\begin{figure*}[h!]
\begin{center}
\includegraphics[width=1.0\linewidth]{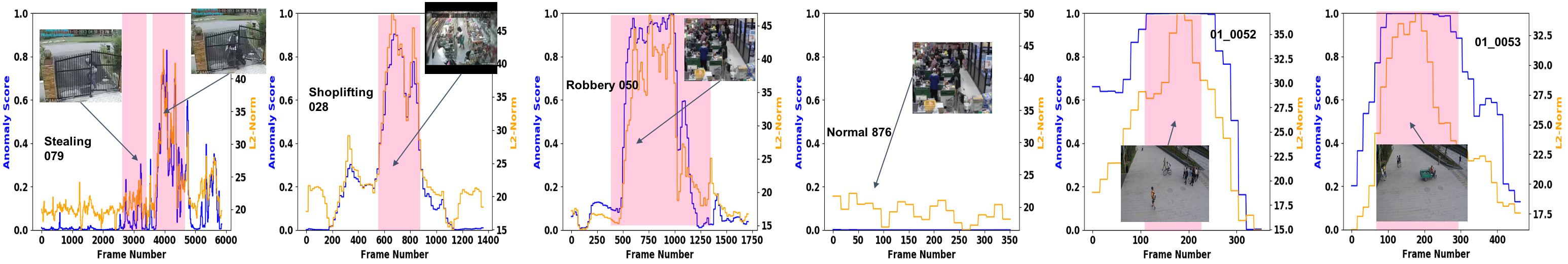}
\end{center}
   \caption{Anomaly scores and feature magnitude values of our method on UCF-Crime (\textit{stealing079},\textit{shoplifting028}, \textit{robbery050} \textit{normal876}), and ShanghaiTech (\textit{01\_0052}, \textit{01\_0053}) test videos.
   Pink areas indicate the manually labelled abnormal events. }
\label{fig:ucf_example}
\end{figure*}

\subsection{Results on UCSD-Peds}

We showed the result on UCSD-Ped2 in Tab.~\ref{tab:ucsd}, with TSN-Gray and I3D-RGB features, respectively. Our approach surpasses the previous SOTA~\cite{zhong2019graph} by a large 3.2\% with the same TSN-Gray features. Finally, we achieves the best 98.6\% mean AUC, surpassing Sultani et al.~\cite{sultani2018real} by 6.3\%, using the same I3D features.

\begin{table}[htbp]
\centering
\scalebox{0.75}{
\begin{tabular}{@{}c|c|c@{}}
\toprule\hline
Method         & Feature    & AUC  (\%)       \\ \hline\hline
GCN-Anomaly \cite{zhong2019graph}   & TSN-Flow            & 92.8  \\
GCN-Anomaly \cite{zhong2019graph}   & TSN-Gray             & 93.2   \\
Sultani et al.*\cite{sultani2018real}   & I3D RGB  &   92.3  \\ \hline
Ours        & TSN-Gray      & \textcolor{blue}{\textbf{96.5}}   \\
Ours        & I3D-RGB   & \textcolor{red}{\textbf{98.6}}  \\ \hline\bottomrule
\end{tabular}%
}
\caption{Comparison of AUC performance with other SOTA weakly-supervised methods on UCSD Ped2. * indicates we retrain the method in~\cite{sultani2018real} using I3D features. Best result in \textcolor{red}{\textbf{red}} and second best in \textcolor{blue}{\textbf{blue}}.}
\label{tab:ucsd}
\end{table}

\subsection{Sample Efficiency Analysis}

We investigate the sample efficiency of our method by looking into its performance w.r.t. the number of abnormal videos used for training on ShanghaiTech. We reduce the number of abnormal training videos from the original 63 videos down to 25 videos, with the normal training videos and test data fixed. The MIL method in ~\cite{sultani2018real} is used as a baseline. For a fair comparison, the same I3D features are used in both methods,  
and average AUC results ((computed from three runs using different random seeds)) are shown in Fig.~\ref{fig:num_abnormal}. As expected, the performance of both our method and Sultani et al.~\cite{sultani2018real} decreases with decreasing number of abnormal training videos, but the decreasing rate of our model is smaller that of than Sultani et al.~\cite{sultani2018real}, indicating the robustness of our RTFM. Remarkably, our method using only 25 abnormal training videos outperforms~\cite{sultani2018real} using all 63 abnormal videos by about 3\%, i.e., although our method uses 60\% less labelled abnormal training videos, it can still outperform Sultani et al.~\cite{sultani2018real}. This is because RTFM performs better recognition of the positive instances in the abnormal videos, and as a result, it can leverage the same training data more effectively than a MIL-based approach~\cite{sultani2018real}. Note that we retrain Sultani et al.'s method using the same I3D features.  

\begin{figure}[h!]
\begin{center}
\includegraphics[width=0.95\linewidth]{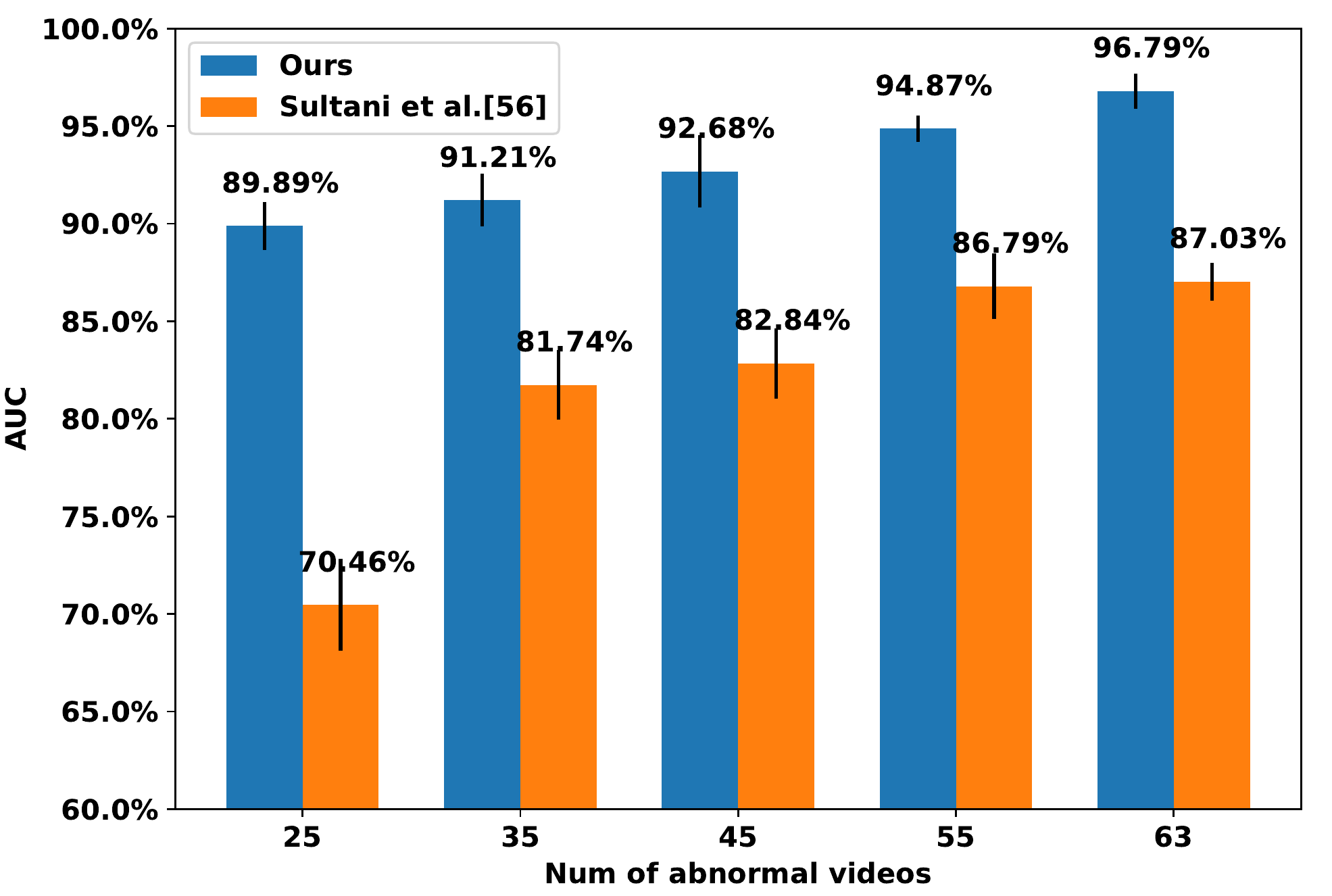}
\end{center}
   \caption{AUC w.r.t. the number of abnormal training videos. 
   }
\label{fig:num_abnormal}
\end{figure}

\subsection{Subtle Anomaly Discriminability}

We also examine the ability of our method to detect subtle abnormal events on the UCF-Crime dataset, by studying the AUC performance on each individual anomaly class. 
The models are trained on the full training data and we use~\cite{sultani2018real} as baseline, and results are shown in Fig.~\ref{fig:class_AUC}. Our model shows remarkable performance on human-centric abnormal events, even when the abnormality is very subtle. Particularly, our RTFM method outperforms Sultani et al.~\cite{sultani2018real} in 8 human-centric anomaly classes (i.e., arson, assault, burglary, robbery, shooting, shoplifting, stealing, vandalism), significantly lifting the AUC performance by 10\% to 15\% in subtle anomaly classes such as burglary, shoplifting, vandalism. This superiority is supported the theoretical results of RTFM that guarantee a good separability of the positive and negative instances. For the arrest, fighting, road accidents and explosion classes, our method shows competitive performance to~\cite{sultani2018real}. Our model is less effective in the abuse class because this class contains overwhelming human-centric abuse events in the training data but its testing videos contain animal abuse events only.

\begin{figure}[h!]
\begin{center}
\includegraphics[width=1.0\linewidth]{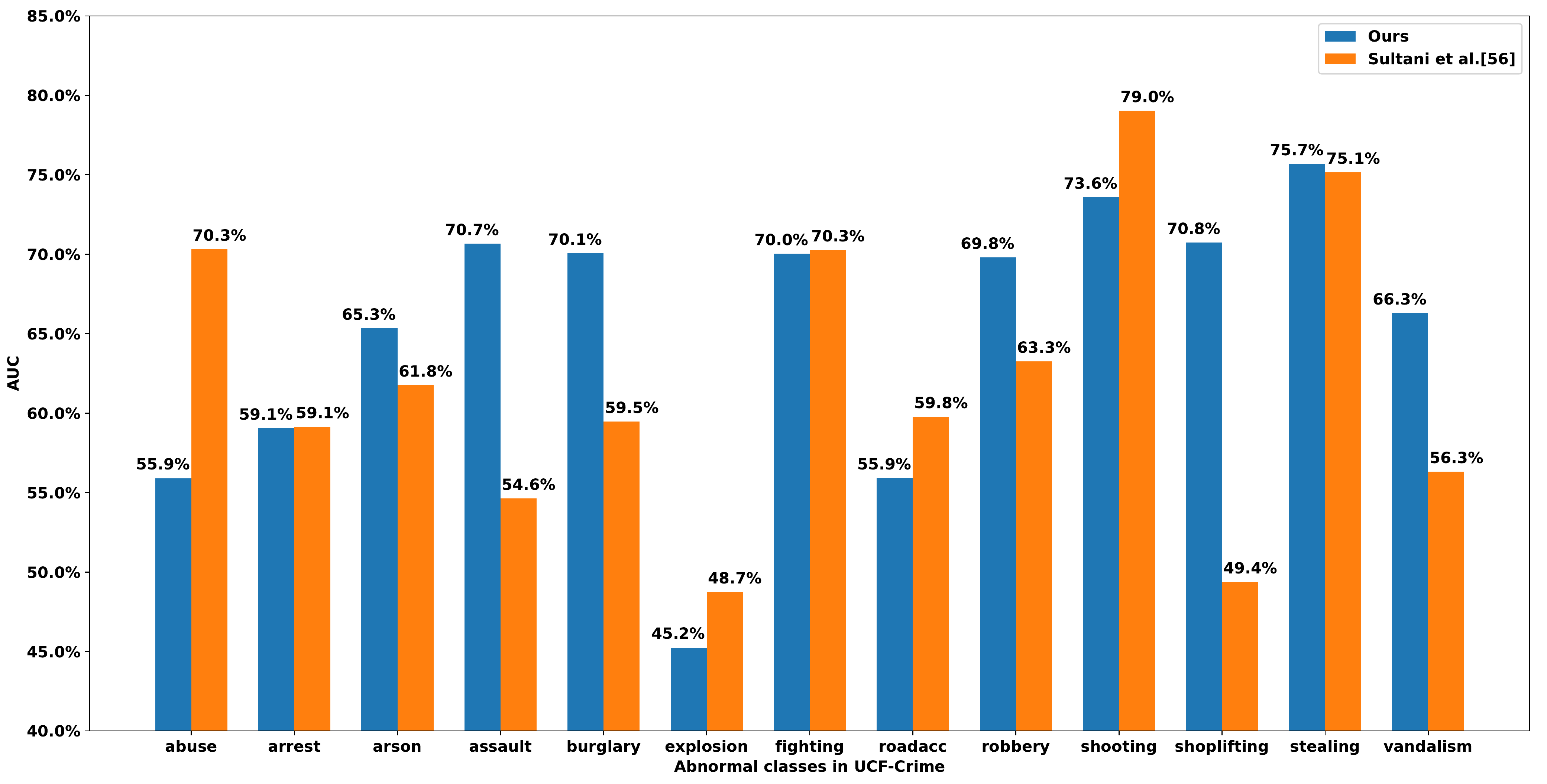}
\end{center}
  \caption{AUC results w.r.t. individual classes on UCF-Crime.
  }
\label{fig:class_AUC}
\end{figure}

\subsection{Ablation Studies}


We perform the ablation study on ShanghaiTech and UCF Crime with I3D features, as shown in Tab.~\ref{tab:ablation}, where the temporal feature mapping function $s_{\theta}$ is decomposed into PDC and TSA, and FM represents the feature magnitude learning from Sec. \ref{sec:top_k_MIL}. 
The baseline model
replaces PDC and TSA with a $1 \times 1$ convolutional layer and is trained with the original MIL approach as in~\cite{sultani2018real}. The resulting baseline achieves only 85.96\% AUC on ShanghaiTech and 77.32\% AUC on UCF Crime (a result similar to the one in~\cite{sultani2018real}).
By adding PDC or TSA, the AUC performance is boosted to 89.21\% and 91.73\% on ShanghaiTech and 79.32\% and 78.96\% on UCF, respectively. When both PDC and TSA are added, the AUC result increases to 92.32\% and 82.12\% for the two datasets, respectively. This indicates that PDC and TSA contributes to the overall performance, and they also complement each other in capturing both long and short-range temporal relations.
When adding only the FM module to the baseline, the AUC substantially increases by over 7\% and 4\% on ShanghaiTech and UCF Crime, respectively, indicating that our feature magnitude learning considerably improves over the original MIL method as it enables better exploitation of the labelled abnormal video data. Additionally, combining either PDC or TSA with FM helps further improve the performance. Then, the full model RTFM can achieve the best performance of 97.21\% and 84.30\% on the two datasets.  
An assumption made in theoretical motivation for RTFM is that the mean feature magnitudes for the top-$k$ abnormal feature snippets is larger than the ones for normal snippets.
We measure that on the testing videos of UCF-Crime and the mean magnitude of the top-$k$ snippets from abnormal videos is 53.4 and for normal, it is 7.7. This shows empirically that our our assumption for Theorem~\ref{thm:expected_separability} is valid and that RTFM can effectively maximise the separability between normal and abnormal video snippets.  This is further evidenced by the mean classification scores of 0.85 for the abnormal snippets and 0.13 for the normal snippets.

\begin{table}[htbp]
\centering
\scalebox{0.75}{
\begin{tabular}{cccc|cc}
\toprule\hline
Baseline & PDC & TSA & FM &  AUC (\%) - Shanghai  &  AUC (\%) - UCF\\ \hline \hline
\checkmark       &           &     &                    & 85.96 &77.39                    \\
\checkmark      & \checkmark         &     &               & 89.21   & 79.32                   \\
\checkmark       &    & \checkmark  &             & 91.73 & 78.96    \\
\checkmark       & \checkmark    & \checkmark       &               &92.32        & 82.12             \\ \hline
\checkmark       &    &   &\checkmark            & 92.99  & 81.28   \\
\checkmark       &    & \checkmark  & \checkmark            & 94.63   & 82.97  \\
\checkmark       & \checkmark    &   & \checkmark           & 93.91  & 82.58   \\\hline
\checkmark       & \checkmark    & \checkmark  & \checkmark       & 
97.21 & 84.30\\ \hline \bottomrule
\end{tabular}%
}
\caption{Ablation studies of our method on ShanghaiTech and UCF-Crime.}
\label{tab:ablation}
\end{table}

\subsection{Qualitative Analysis}

In Fig.~\ref{fig:ucf_example}, we show the anomaly scores produced by our MIL anomaly classifier for diverse test videos from UCF-Crime and ShanghaiTech. Three anomalous videos and one normal video from UCF-Crime are used (\textit{stealing079}, \textit{shoplifting028}, \textit{robbery050} and \textit{normal876}).  
As illustrated by the $\ell_2$-norm value curve (i.e., orange curves), our FM module can effectively produce a small feature magnitude for normal snippets and a large magnitude for abnormal snippets. Furthermore, our model can successfully ensure large margins between the anomaly scores of the normal and abnormal snippets (i.e., blank and pink shadowed areas, respectively). Our model is also able to detect multiple anomalous events in one video (e.g., \textit{stealing079}), which makes the problem more difficult. Also, for the anomalous events  $stealing$ and $shoplifting$, the abnormality is subtle and barely seen through the videos, but our model can still detect it. 
We also show the anomaly scores and feature magnitudes produced by our model for \textit{01\_0052} and \textit{01\_0053} from ShanghaiTech (last two figures in Fig.~\ref{fig:ucf_example}). Our model can effectively yield large anomaly scores for the anomalous event of vehicle entering in these two scenes.








\section{Conclusion}

We introduced a novel method, named RTFM, that enables top-$k$ MIL approaches for weakly supervised video anomaly detection. RTFM learns a temporal feature magnitude mapping function that 1) detects the rare abnormal snippets from abnormal videos containing many normal snippets, and 2) guarantees a large margin between normal and abnormal snippets. This improves the subsequent MIL-based anomaly classification in two major aspects: 1) our RTFM-enabled model learns more discriminative features that improve its ability in distinguishing complex anomalies (e.g., subtle anomalies) from hard negative examples; and 2) it also enables the MIL classifier to achieve significantly improved exploitation of the abnormal data. These two capabilities respectively result in better subtle anomaly discriminability and sample efficiency than current SOTA MIL methods. They are also the two main drivers for our model to achieve SOTA performance on all three large benchmarks.

{\small
\bibliographystyle{ieee_fullname}
\bibliography{egbib}
}

\newpage

\beginsupplement
\setcounter{section}{0}

\setcounter{equation}{0}
\setcounter{figure}{0}
\setcounter{table}{0}
\setcounter{page}{1}
\makeatletter
\renewcommand{\theequation}{S\arabic{equation}}
\renewcommand{\thefigure}{S\arabic{figure}}
\newpage
\section{Supplementary Material}
\subsection{Theoretical Motivation of RTFM}

\begin{theorem}[Expected Separability Between Abnormal and Normal Videos]
\label{thm:expected_separability}
Assuming that $\mathbb{E}[\|\mathbf{x}^+\|_2] \ge \mathbb{E}[\| \mathbf{x}^- \|_2]$, where $\mathbf{X}^+$ has $\mu$ abnormal samples and $(T-\mu)$ normal samples, where $\mu \in [1,T]$, and $\mathbf{X}^-$ has $T$ normal samples.
Let $D_{\theta,k}(.)$ be the random variable from which the separability scores $d_{\theta,k}(.)$ of Eq.3 in the main paper are drawn~\cite{li2015multiple}.
\begin{enumerate}
    \item If  $0 < k < \mu$, then
    $$0 \le \mathbb{E}[D_{\theta,k}(\mathbf{X}^+,\mathbf{X}^-)] \le \mathbb{E}[D_{\theta,k+1}(\mathbf{X}^+,\mathbf{X}^-)].$$
    \item For a finite $\mu$, then
    $$ \lim_{k \to \infty} \mathbb{E}[D_{\theta,k}(\mathbf{X}^+,\mathbf{X}^-)] = 0.$$
\end{enumerate}
\end{theorem}
\begin{proof}
\begin{equation}
\begin{split}
\mathbb{E}&[D_{\theta,k}(\mathbf{X}^+,\mathbf{X}^-)] = \mathbb{E}[g_{\theta,k}(\mathbf{X}^+)]-\mathbb{E}[g_{\theta,k}(\mathbf{X}^-)] \\
&=p^+_{k}(\mathbf{X}^+) \mathbb{E}[\|\mathbf{x}^+\|_2] + p^-_{k}(\mathbf{X}^+) \mathbb{E}[\| \mathbf{x}^-\|_2] 
- \mathbb{E}[\| \mathbf{x}^- \|_2]
\end{split}
\label{eq:proof_expected}
\end{equation}

\begin{enumerate}
\item Trivial given that
$\mathbb{E}[\|\mathbf{x}^+\|_2] \ge \mathbb{E}[\|\mathbf{x}^-\|_2]$ and that $p^+_{k+1}(\mathbf{X}^+) > p^+_{k}(\mathbf{X}^+)$ for $0 < k < \mu$
\item Trivial given that as $\mu$ is finite, $\lim_{k \to \infty }p^+_{k}(\mathbf{X}^+) = 0$.
\end{enumerate}
\end{proof}

\textbf{Intuition of feature magnitude}: Assuming the expected magnitude of abnormal samples is larger than of normal samples, we can derive Thm. 3.1 that proves that the expected feature magnitude-based separability score between normal and abnormal videos grows for $0<k<\mu$ and reduces to zero for $k \to \infty$. Hence, to use Thm. 3.1, we need to enforce larger magnitude for abnormal features using our proposed RTFM. The similarity between the theoretical and empirical curves in Fig.\ref{fig:k_fig}(left) is evidence of the soundness of Thm. 3.1.

\subsection{Multi-scale Temporal Feature Learning}

Our proposed multi-scale temporal network (MTN) captures the multi-resolution local temporal dependencies and the global temporal dependencies between video snippets, as displayed in Fig.~\ref{fig:MTN}.

\begin{figure}[t!]
\begin{center}
\includegraphics[width=1.0\linewidth]{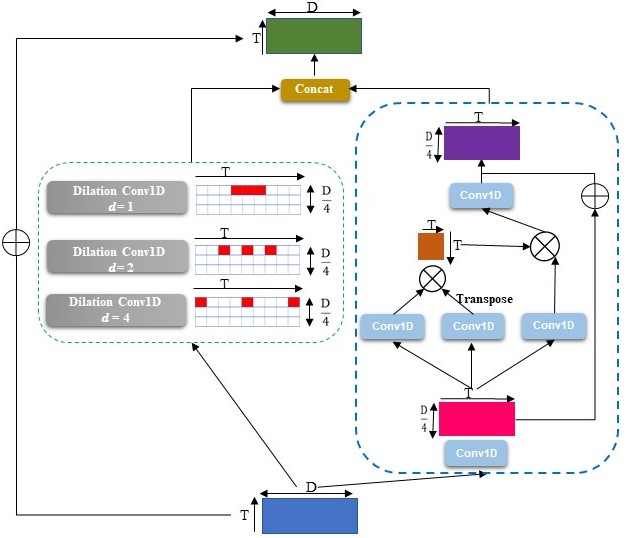}
\end{center}
  \caption{Our proposed MTN consists of two modules. The module on the left uses the pyramid dilated convolutions to capture the local consecutive snippets dependency over different temporal scales. The module on the right relies on a self-attention network to compute the global temporal correlations. The features from the two modules are concatenated to produce the MTN output.}
\label{fig:MTN}
\end{figure}

\subsection{Computational Efficiency}

We investigate if our system can run in real time. During inference, our method processes a 16-frame clip in 0.76 seconds on a Nvidia 2080Ti--this time includes the I3D extraction time. This indicates that our system can achieve good real-time detection in real-world applications. 

\subsection{Temporal Dependency}
Temporal Dependency has been explored in~\cite{liu2018future,luo2017revisit,liu2019margin,zhong2019graph,Wu2020not,kratz2009anomaly,xu2014video}. 
In anomaly detection, traditional methods~\cite{kratz2009anomaly,xu2014video} convert consecutive frames into handcrafted motion trajectories to capture the local consistency between neighbouring frames. 
Diverse temporal dependency modelling methods have been used in deep anomaly detection approaches, such as stacked RNN~\cite{luo2017revisit}, temporal consistency in future frame prediction~\cite{liu2018future}, and convolution LSTM~\cite{liu2019margin}.
However, these methods capture short-range fixed-order temporal correlations only with single temporal scale, ignoring the long-range dependency from all possible temporal locations and the events with varying temporal length. GCN-based methods are explored in~\cite{zhong2019graph,Wu2020not} to capture the long-range dependency from snippets features, but they are  inefficient and hard to train.
By contrast, our proposed module combines 
PDC~\cite{yu2015multi} and TSA~\cite{wang2018non} on the temporal dimension to seamlessly and efficiently incorporate both the long and short-range temporal dependencies into our temporal feature ranking loss (See Sec.~\ref{sec:temporal}). 

\begin{figure}[h!]
\begin{center}
\includegraphics[width=0.98\linewidth]{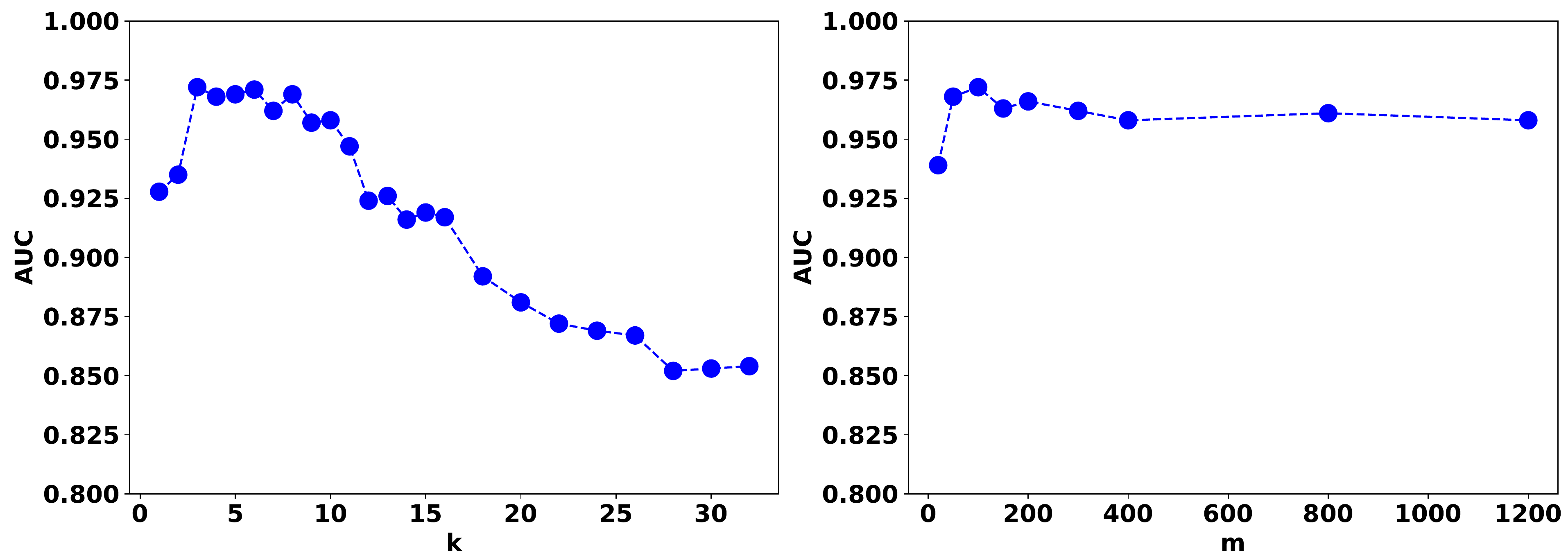}
\end{center}
 \caption{AUC w.r.t. top-$k$ (\textbf{Left}) and the margin $m$ (\textbf{Right}).  
 }
\label{fig:k_fig}
\end{figure}

\subsection{Ablations for $k$ and $m$}
We show the AUC results as a function of top-$k$ and margin $m$ values on ShanghaiTech in Fig.\ref{fig:k_fig}. Consistent to our theoretical analysis, the performance of our model peaks at a sufficiently large $k$, flattens at around $k \approx \mu$ and then drops with increasing  $k$ (Fig.\ref{fig:k_fig}(left)). It is also robust to a large range of $m\in[50,1200]$ with a stable AUC in $[93\%,96\%]$ (Fig.\ref{fig:k_fig}(right)).

\end{document}